%% file: bandits-pkdd2018.tex
\documentclass[runningheads]{llncs}
\usepackage{graphicx}
\usepackage{booktabs} 
\usepackage[ruled,vlined]{algorithm2e}
\usepackage{todonotes}
\usepackage{amsfonts}
\usepackage{amsmath}

\usepackage{amsthm}
\usepackage[style=numeric-comp, backend=bibtex]{biblatex}
\usepackage[breaklinks=true]{hyperref}
\usepackage{cleveref}
\usepackage{breakcites}
\DeclareMathOperator*{\argmax}{arg\,max} 

\theoremstyle{assumption}
\newtheorem{assumption}{Assumption}

\bibliography{bandits-pkdd2018} 

\begin{document}
\title{Exploring Partially Observed Networks with  Nonparametric Bandits}
%
%
\author{Kaushalya Madhawa\inst{1} \and Tsuyoshi Murata\inst{1}}
\authorrunning{K. Madhawa and T.Murata}
%
\institute{Tokyo Institute of Technology, Ookayama, Meguro-ku, Tokyo 152-8552, Japan
\email{kaushalya@net.c.titech.ac.jp, murata@c.titech.ac.jp}
\url{http://www.net.c.titech.ac.jp}}
\maketitle              
\begin{abstract}
Real-world networks such as social and communication networks are too large to be observed entirely. Such networks are often \textit{partially observed} such that network size, network topology, and nodes of the original network are unknown. In this paper we formalize the \textit{Adaptive Graph Exploring} problem. We assume that we are given an incomplete snapshot of a large network and additional nodes can be discovered by querying nodes in the currently observed network. The goal of this problem is to maximize the number of observed nodes within a given query budget. Querying which set of nodes maximizes the size of the observed network? We formulate this problem as an exploration-exploitation problem and propose a novel nonparametric multi-arm bandit (MAB) algorithm for identifying which nodes to be queried. Our contributions include: (1) $i$KNN-UCB, a novel nonparametric MAB algorithm, applies $k$-nearest neighbor UCB to the setting when the arms are presented in a vector space. (2) provide theoretical guarantee that $i$KNN-UCB algorithm has sublinear regret and (3) applying $i$KNN-UCB algorithm on synthetic networks and real-world networks from different domains, we show that our method discovers up to 40\% more nodes compared to existing baselines. 

\keywords{network exploration, network search, multi armed bandits}
\end{abstract}

\input{bandits-pkdd2018-body}

\printbibliography

\end{document}

%% file: bandits-pkdd2018-body.tex
\section{Introduction}

Interactions among different entities in many real-world complex systems are often represented by networks, where the entities are represented by nodes and the interactions among them are represented as links between entities. For example, the information contained in online social networks proved to be valuable in advertising applications such as finding influential users to targeted marketing. Data acquisition is done using Application Programming Interfaces (APIs) offered by respective social networking services. Using these APIs is often time consuming and the number of nodes (e.g., profiles) that can be queried within a given time is restricted. A poorly constructed incomplete network will lead to inaccurate findings. This highlights the importance of acquiring more information as possible using a limited number of queries.

Here, we provide an overview of \textit{Adaptive Graph Exploration} problem. We formally define it in \autoref{sec:method}. Suppose we are given a partially observed network. For instance, a sample of a social network collected by a researcher. Since we do not know how this sample is obtained, only way to enhance this sample is by acquiring data belonging to the unseen portion of the network. We use the term \textit{probing} to refer to \textit{querying} a node to retrieve information about it and its neighborhood. As an example, probing a node of a social network corresponds to obtaining information about a profile and its friends (or followers) using an API or a web service. Several rounds of probing updates the sample with new nodes and links found in the neighborhood of queried nodes. The number of times the network can be probed is restricted by a \textit{probing budget}. Thus, the goal is to enhance the observed graph as much as possible within the \textit{probing budget}.

Two approaches have been proposed to solve the problem of reducing the incompleteness of partially observed networks. First approach involves inferring properties of the unseen part of the network using knowledge of the sample. Such methods infers the missing information by fitting a model of network structure to the observed part \cite{kim2011network}. However, this is not practical for real-world networks as such methods require more structural information about the complete network.  Second approach is acquiring more information by probing as we propose in this paper. Existing heuristic algorithms such as maximum observed degree (MOD) probing and maxreach \cite{soundarajan2016maxreach} require the sample to be obtained in a certain way (e.g., uniform edge sampling). In section \ref{sec:experiments} we show that existing probing algorithms can not be generalized for incomplete networks obtained by different sampling techniques. Furthermore, many real world networks consist of communities, densely connected regions of nodes. Heuristic probing algorithms get stuck inside communities, making them worse than probing a node in random.

\subsubsection*{Our Work.}
A high level overview of the proposed adaptive probing algorithm is illustrated in \autoref{fig:pipeline}. The probing pipeline consists of two major steps, obtaining a feature representation of the observed network and a model which predicts the reward a node will reveal (e.g., the true degree of that node) based on its feature vector. The key assumption of using a learning model is that nodes with similar features in the observed network will result in similar rewards. Our choice of graph features is motivated by work on inferring structural role \cite{henderson2012rolx} and social status \cite{zhao2013inferring} of nodes in social networks. 

\begin{figure}
	\centering
	\includegraphics[width=\textwidth]{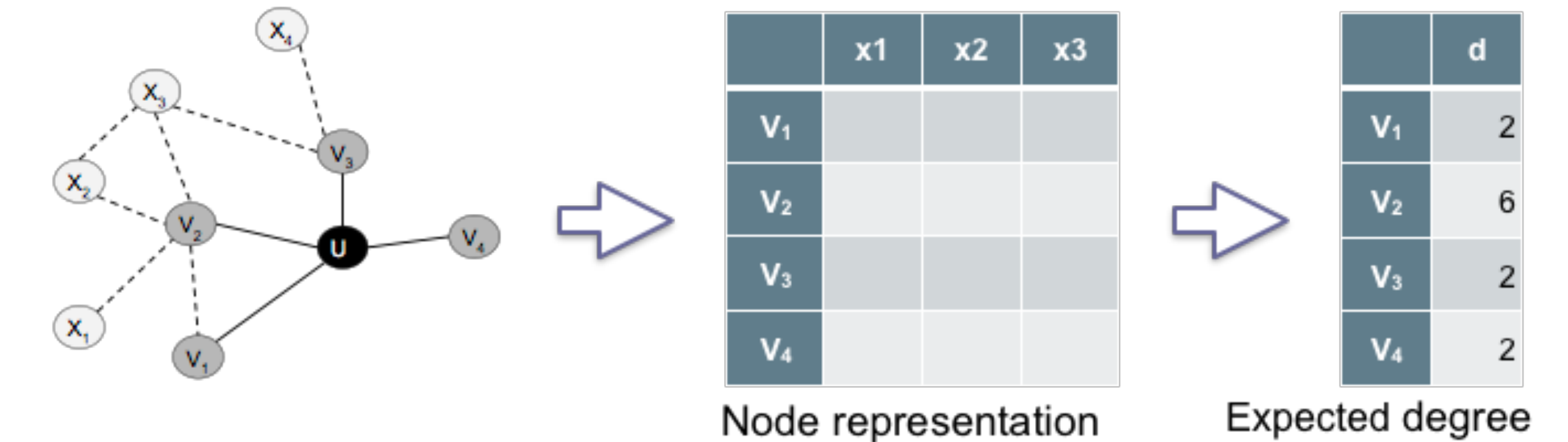}
	\caption{prediction pipeline}
	\label{fig:pipeline}
\end{figure}

One property which makes estimation of rewards different from a normal prediction problem is that our training data is accumulated over the process of probing. Probing nodes with similar features all the time may result in sub-optimal results. This situation is known in reinforcement learning literature as \textit{exploration-exploitation} trade off. Multi-armed bandits \cite{robbins1952some} is a generic way to approach real-world exploitation-exploration problems. In this context, exploitation corresponds to selecting the node which has the largest expected reward and exploration corresponds to selecting some other node for probing. 

Our contributions are threefold:
\begin{enumerate}
	\item A generic approach for enhancing partially observed networks which does not require any prior knowledge about the network.
	\item A novel non-paramteric UCB algorithm ($i$KNN-UCB) to solve the multi-armed bandit problem (MAB) when the arms are represented in a vector space. \footnote{source code available at \url{https://bitbucket.org/kau_mad/bandits/src/pkdd2018/}}
	\item Using $i$KNN-UCB algorithm on synthetic networks and real-world networks from different domains, we demonstrate that our proposed method performs significantly better than existing methods. \footnote{source code available at \url{https://bitbucket.org/kau_mad/net_complete/src/pkdd2018}}
\end{enumerate}

The rest of the paper is structured as following. In \autoref{sec:related work}, we provide an extensive review of related work. \autoref{sec:method} starts with the problem definition and describes our approach in detail. \autoref{sec:experiments} explains the experimental setup and the data sets being used. Then, in \autoref{sec:results} we present empirical evaluations of our bandit algorithm using real-world networks as well as synthetic networks. Finally, \autoref{sec:conclusions} concludes with a brief discussion of the bandit approach and a few promising directions as future work. 

\section{Related Work} \label{sec:related work}

\subsection{Network Crawling and Sampling}
Although this problem looks similar to network crawling and sampling, objective of most sampling algorithms is to select a representative subset of the nodes (or edges) when the entire network is accessible \cite{ahmed2014network}.
 In contrast, we are improving a given incomplete network and we have no knowledge of how the sample is being obtained. Particularly, snowball sampling \cite{lee2006statistical} can be used when the information about the complete network is not accessible. But it suffers from the same drawbacks as of heuristic algorithms; it does not adapt as the observed information updates. As another related problem, link prediction \cite{liben2007link} can predict missing links on a network, but not missing regions of nodes. The only way to enhance the observed sample is by iteratively querying observed nodes and adding their neighboring nodes to the sample.

\subsection{Active Search}
Active search on graphs \cite{wang2013active,bilgic2010active} is another related problem with the objective of finding as much \textit{target nodes}  as possible possessing a given property. Most of the previous work relating to this problem assume that the complete graph is observable and any node can be queried to find its label \cite{ma2015active}. If only an incomplete view is available, relying only on the observed information may not obtain the best possible reward. In addition to \textit{exploitation} of the best option according to available information, \textit{exploration} of other possible options is performed to achieve better rewards. A common approach to finding a balance between exploitation vs exploration trade-off is formulating it as a multi-armed bandit problem (MAB) \cite{mahajan2008multi}.  SN-UCB1\cite{bnaya2013social} and NETEXP\cite{singla2015information} are such MAB based active search algorithms proposed for partially observed networks. Probing a node in NETEXP reveals 2-hop neighborhood, which is not true for real world social networks. SN-UCB1 does not provide a significant improvement over the existing heuristic methods. \textcite{soundarajan2017varepsilon} recently proposed $\epsilon$-WGX, a multi-armed bandit approach to solve Active Edge Probing (AEP) problem in incomplete networks. Though AEP looks similar, it is fundamentally different from ours as a node can be probed multiple times and only one neighboring edge is revealed in each probe.

\section{Proposed Bandit Based Probing Method} \label{sec:method }
We start this section with the formal definition of the problem. Then we describe the main components of this work and the multi-armed bandit algorithm in detail.

\label{sec:method}
\subsection{Problem Definition}
Suppose there is a large unweighted undirected graph $G$ which can not be observed fully, but only a partially observed network $G'$ is available. We denote the initial incomplete network as  $G'_0$. Our goal is to grow this network by probing any of the observed nodes at each time step. Using this notation we denote the observed network at time $t$ as $G'_t$. \autoref{tab:notations} lists the notation that we will be using in this section.

\begin{table}[]
\centering
\caption{Table of notations}
\label{tab:notations}
\begin{tabular}{ll}
\toprule
Symbol & Definition \\
\midrule
$G(V, E)$ & original network \\
$G'_t(V'_t, E'_t)$ & observed network at time $t$ \\
$K_t$ & set of candidate nodes at time $t$ \\
$T$ & probing budget \\
\bottomrule
\end{tabular}
\end{table}

\begin{definition}
  \textbf{Probing} a node reveals all links incident to it and the identity of its neighboring nodes. 
\end{definition}
 
The number of times we are allowed to probe the network is constrained by the \textbf{probing budget} ($T \in \mathbb{Z}$)

\begin{figure}
\centering
\includegraphics[width=8cm]{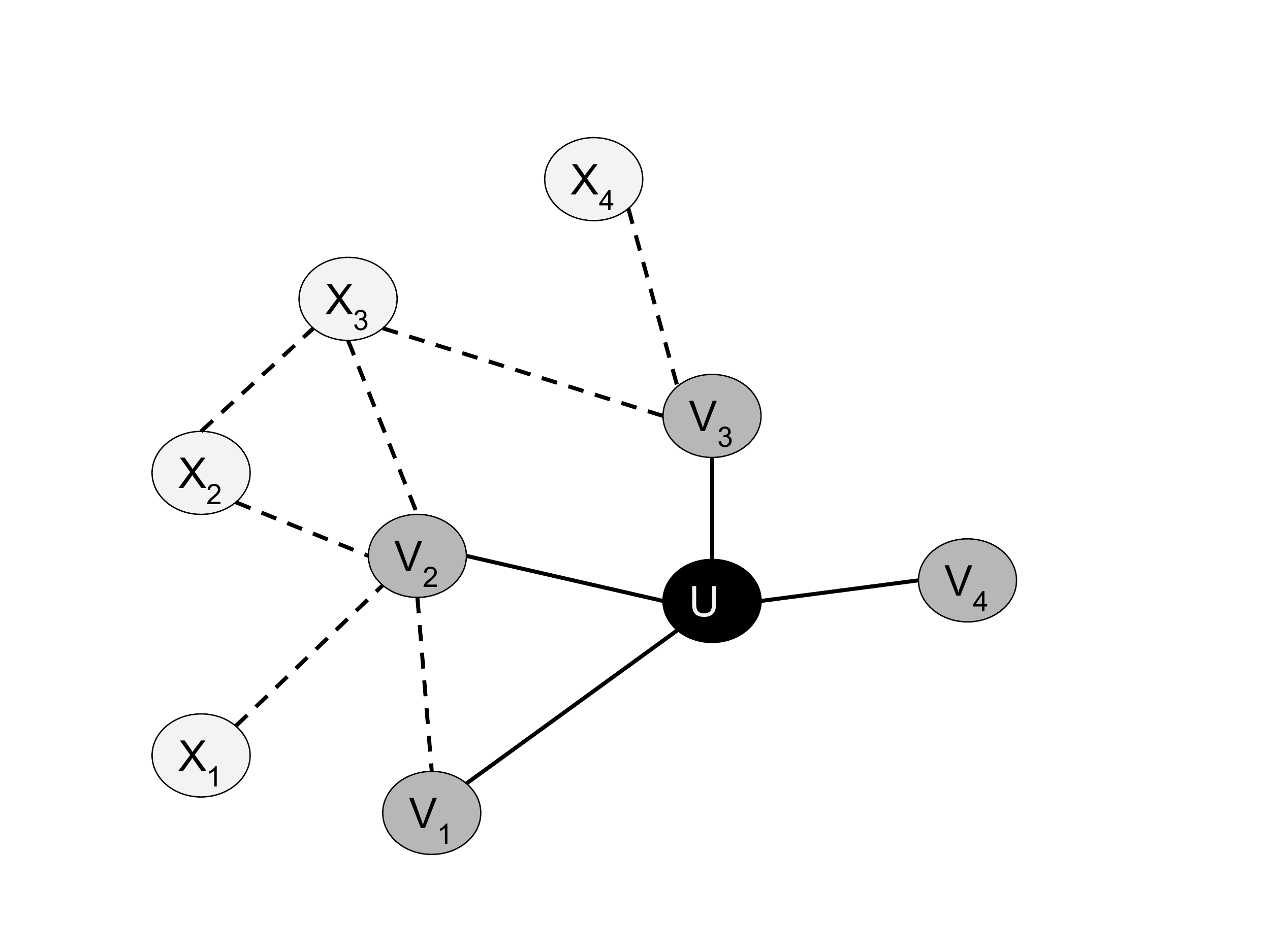}
\caption{Example of an incomplete network. The black node $U$ is probed and gray nodes {$V_1,\cdots, V_4$} are observed. The white nodes {$X_1,\cdots, X_4$} exist in the original network $G$, are yet to be observed.}
\label{fig:network}
\end{figure}

\begin{definition}
  At time $t$, a node in the original network $G$ can belong to any of the following three sets. 
  \begin{enumerate}
    \item \textbf{unobserved}: existence of these nodes is not visible to the algorithm.
	\item \textbf{observed}: these nodes exist in both $G$ and $G'_t$, but has not being probed.
    \item \textbf{probed}: the algorithm knows about these nodes and their neighboring nodes.
 \end{enumerate}
 \label{def:problem}
\end{definition}

\autoref{fig:network} illustrates an example incomplete network. We use bold lines to denote observed links and dash lines to denote unobserved links at the given moment. Even though nodes $V_1$ and $V_2$ are observed when node $U$ is probed, [$V_1, V_2$] link is not observed because neither nodes are probed.

An observed node can either be probed or not probed at the moment. Any observed node which is not probed is considered as a candidate for probing. Hence, we refer such nodes as \textit{candidate nodes}. At the beginning, all the nodes in the given sample are candidate nodes.  Probing a candidate node reveals a \textit{reward} (eg. true degree of a node). Our goal is iteratively selecting $\textbf{b}$ candidate nodes that maximizes the cumulative reward (i.e., number of observed nodes).

\subsection{Calculation of expected reward of candidate nodes}
Instead of using a heuristic metric to choose a candidate node for probing in each time step, we treat this problem as a learning problem. Similar to an active exploration algorithm, our proposed solution consists of three high level steps \cite{pfeiffer2014active}: probing, learning, and prediction. \textit{Probing} a node results in additional information about the observed network. Information about the currently observed network is leveraged to \textit{learn} a predictive model which \textit{predicts} the expected reward of a given candidate node in future. Our approach assumes that candidate nodes with similar structural neighborhoods will result in similar rewards.


Suppose that the feature vector of a candidate node $j$ at time $t$ is $x_{j, t} \in \mathbb{R}^d$.  The learner probes node $j$ at time $t$ and observes the following reward
\begin{equation*}
r_{j,t} = f(x_{j,t}) + \zeta_t,
\end{equation*}
where $f:\mathcal{X} \rightarrow \mathbb{R}$ gives the expected reward of a given node and  $\zeta_t$ is sub-gaussian white noise with mean 0 and variance $\sigma^2$.

\theoremstyle{assumption}
\begin{assumption} \label{asm:lipschitz}
	(Lipschitz condition):
	There exists a constant $L$ such that $|f(x) - f(x')| \leq L \cdot \mathcal{D}(x-x')$ for all $x$,$x'$ $\in \mathcal{X}$. $\mathcal{D}$ is a metric which defines the ``distance" between two vectors $x$ and $x'$.
\end{assumption}

Assumption \ref{asm:lipschitz} expresses that nodes which are similar in terms of their feature vectors will have similar rewards. In the next section, we describe in detail how we formulate this problem as a multi-armed bandit problem. 

\subsection{Bandit Algorithm}
\subsubsection{Problem Setting}
In the classical contextual multi-armed bandit problem, an agent selects one of the $K$ arms (or actions) at each time step and observes a reward depending on the chosen action. In this setting, each arm is assumed to be independent, the rewards are drawn randomly from a probability distribution that is specific to each arm. The goal of the agent is to play a sequence of actions which maximizes the cumulative reward it receives within a given number of time steps. 

Selecting a node from the set of candidate nodes at time step $t$ for probing is similar to pulling an arm in a multi-armed bandit problem. However, the classical notion of K-armed bandit problem assumes that the set of $K$ arms would not change over time and requires each arm to be played several times. In contrast, the set of candidate nodes change as probings  occur over time. And more importantly, a node can not be probed for a second time.

As independent assumption does not hold in our problem setting, it is more suitable to express it as a \textit{structured bandits} problem, in which reward distributions of arms are not independent, but interrelated. In structured bandit problem, the agent deduces relationship between arms based on some $d$-dimensional feature vector $x_a \in \mathbb{R}^d$ assigned to an arm $a$.  

\subsubsection{KNN-UCB algorithm for structured bandits}
\textit{Linear bandits}\cite{rusmevichientong2010linearly, Dani2008} the simplest among such models, assumes the reward is linearly dependent on feature vectors and computes the expected reward of an arm by the inner product of its feature vector and a parameter vector $\theta$. But real data often exhibits more complicated relationships than a linear one. Hence, we choose $k$-nearest neighbor (k-NN) regression to estimate the expected reward of arms. We adapt \textcite{guan2018nonparametric}'s k-armed KNN-UCB algorithm to the structured setting. Upper confidence bound \cite{auer2002using} (UCB) algorithms incorporate an exploration term by calculating a confidence bound for each arm and choose the action corresponding to the largest confidence bound. 

We define $k$-nearest neighbor upper confidence bound ($i$KNN-UCB) rule as
\begin{equation}
\label{eqn:knnucb}
a_t = \argmax_i \hat{f} (x_i) + \alpha \cdot \sigma(x_i)
\end{equation}
where $\alpha > 0$ is a constant determining the amount of exploration.

\begin{definition} \label{def:knn}
	Let the $k$-NN radius of $x \in \mathcal{X}$ be $r_k(x) = inf\{r: |B(x, r_k(x) \cap X) \geq k|\}$ where $B(x,r) = \{ x \in \mathcal{X} : \mathcal{D}(x, x') \leq r\}  $. $k$-NN set of $x \in \mathcal{X}$ be 
	$\mathcal{N}_k (x) := B(x, r_k(x)) \cap X$.
Expected reward of arm $i$, $\hat{f}(x_i)$ is estimated with weighted $k$-NN regression as
\begin{equation}
\label{eqn:knn}
\hat{f}(x_i) = \frac{1}{k} \sum_{x_j \in \mathcal{N}_k(x_i)} \frac{y_j}{\mathcal{D}(x_i, x_j)}\ ,
\end{equation} 
where $y_j$ is the observed reward for $x_j$ and $\mathcal{D}(x_i, x_j)$ is the euclidean distance between feature vectors $x_i$ and $x_j$.
\end{definition}

We define $\sigma(x)$ as the average distance to points in the k-neighborhood,
\begin{equation}
\label{eqn:sigma}
\sigma(x_i) = \frac{1}{k}\sum_{x_j \in \mathcal{N}_k(x_i)}{\mathcal{D}(x_i, x_j)}\ .
\end{equation}

The term $\sigma(x_i)$ is analogous to the term $T_i(t)$ accounting for the number of times action $i$ has been chosen by the time $t$. The way the network is being probed using $i$KNN-UCB is shown in algorithm \ref{alg:bandit}. 

\begin{algorithm}
	\DontPrintSemicolon
	\LinesNumbered
	\caption{$i$KNN-UCB.}
	\label{alg:bandit}
	\SetKwInOut{Input}{Input}
	\SetKwInOut{Output}{Output}
	\SetKwInput{kwInit}{Initialize}
	\Input{incomplete network $G_0' = (V_0', E_0')$, probing budget $b \in \mathbb{N}$, exploration parameter $\alpha$, $k$, $T_0$}
	\Output{A sequence of \textbf{$b$} nodes to probe}
	\kwInit{\textit{candidate nodes} = $V_0'$}
	
	\For{$t \leftarrow 1$ \KwTo $T$}{
		\uIf{$t \leq T_0$}{sample $\textbf{a}_t$ uniformly from $\mathcal{A}_t$ \;}
		\uElse{
			\For{$i$ in \textit{candidate nodes}}
			{
				calculate the feature vector $x_i$\;
				calculate the estimated reward $\hat{f}(x_i)$ with \cref{eqn:knn} \;
				calculate exploration term $\sigma(x_i)$ with \cref{eqn:sigma} \;
			} 
			find the node $\textbf{a}_t$  corresponding to the largest UCB with \cref{eqn:knnucb}
		}
		probe node $\textbf{a}_t$ in the original graph $\textbf{G}$ and observe the reward $\textbf{r}_{t, a_t}$ \;
		Add neighboring nodes $N_{a_t}$ of node $a_t$ to the incomplete network $\textbf{G}'_{t-1}.$ $(G'_t = G'_{t-1} \cup N_{a_t})$ \;
		remove node $\textbf{a}_t$ from \textit{candidate nodes} \;
	} 
\end{algorithm}

\subsubsection{Regret}
The objective of a bandit algorithm is to select arms so as to maximize the cumulative reward over time. Minimization of total regret, is an equivalent way of expressing maximization of cumulative reward. The regret at iteration $t$ equals to the difference between reward of the ``optimal" arm and the reward of a suboptimal arm. In simple terms, regret is the loss incurred by the policy for not playing the optimal arm all the times. In $T$ iterations, we pull arms $a_1, a_2, \cdots, a_n$ and we observe rewards $r_{a_1, 1}, r_{a_2, 2}, \cdots, r_{a_n, n}$. We use the following notion of regret
\begin{equation*}
\mathcal{R}_T = \sum_{t=1}^T [\max_{a} {r_{a, t}} - r_{a_t, t}]\ .
\end{equation*}

\begin{theorem}
	Let $M > 0$ be an arbitrary constant. Then the regret is sublinear with,
	 	$\mathcal{R}_T \leq M \cdot T^{(1-1/d)}$.
\end{theorem}

\begin{proof}
	The regret for bandits in a continuous feature space is
	\begin{equation} \label{eqn:regret}
	\mathcal{R}_T = \sum_{t=1}^T [\sup_{x_i \in \mathcal{X}} {f(x_{i,t})} - f(x_{a_t, t})] .
	\end{equation}
	
	Let $\sup_{i \in \mathcal{A}_t} {f(x_{i,t})}$ be $f_{\pi_t, t}$
	\begin{equation*}
		\mathcal{R}_T = \sum_{t=1}^T [f(x_{\pi_t, t}) - f(x_{a_t, t})]
	\end{equation*}

 	Using Lipschitz assumption
 	\begin{equation} \label{eqn:regret_lipschitz}
		\mathcal{R}_T  \leq \sum_{t=1}^T \big[  L \cdot \mathcal{D}(x_{\pi_t, t},x_{a_t, t}) \big]
 	\end{equation}
 	
 	\begin{equation} \label{eqn:regret_last}
 	\mathcal{R}_T \leq L \cdot \sum_{t=1}^T \big[ \sup_{x \in \mathcal{X}} r_k(x)  \big]\ .
 	\end{equation}
 	
 	From \textcite{jiang2017rates}, 
 	\begin{equation}
 	\sup_{x \in \mathcal{X}}r_k(x) \leq M_1 \cdot \bigg( \frac{2k}{t} \bigg)^{1/d}
 	\end{equation}
 	where $M_1>0$ is a constant. Using this in \cref{eqn:regret_last} results in
 	\begin{equation}
 	\mathcal{R}_T \leq L \cdot \big[\sum_{t=1}^T M_1 \cdot \bigg( \frac{2k}{t} \bigg)^{1/d} \big]
 	\end{equation}
 	
 	With $M \geq L \cdot M_1$
 	\begin{eqnarray}
 		\mathcal{R}_T  \leq M \int_1^T t^{-1/d} dt
 	  	\\ \leq M \cdot T^{(1-1/d)}\ .
 	\end{eqnarray}
 	Hence, the regret is sub-linear.
\end{proof}

\begin{remark}
	If we select $\alpha \geq L$, we can write \cref{eqn:regret_lipschitz} as 
	\begin{equation}
	\mathcal{R}_T  \leq \sum_{t=1}^T \alpha \cdot \sigma(x_{a_t, t})\ .
	\end{equation}
\end{remark}

\section{Experiments} 
\label{sec:experiments}
 We construct the feature vector $x_j$ of candidate node $j$ as a vector of following features. For each feature, the local neighborhood of node $j$ in the observed graph $G'_t$ is considered. 
\begin{enumerate}
	\item degree centrality 
	\item average degree centrality of its neighbors
	\item median degree centrality of its neighbors
	\item the average percentage of probed neighbors found in the neighborhood
\end{enumerate}
These features are chosen because their effectiveness is shown in previous work on finding structurally similar nodes \cite{henderson2012rolx}.

\subsection{Data}
We use simulated network data as well as publicly available\footnote{http://snap.stanford.edu/data/index.html} real-world data sets of social and information networks. 

\subsubsection{Synthetic data.}
The aim of using synthetic networks is to investigate the behavior of the proposed method on networks with different network configurations. We use two random network models, Barabasi-Albert model (BA) \cite{barabasi1999emergence} and Lancichinetti-Fortunato-Radicchi (LFR) \cite{lancichinetti2008benchmark} benchmark to create networks with different characteristics. All these networks have the same number of nodes ($N=34,546)$, the number of nodes in the HepPh citation network. BA model generates networks with power-law degree distributions. But real-world communication networks possess different properties such as homophily \cite{mcpherson2001birds} which can not be represented by a BA model. We use LFR model to generate networks with community structure. The mixing parameter $\mu$ of LFR model decides the probability of a node linking other nodes belonging to different communities. Low values of $\mu$ will result in dense communities as the chance of having intra-community links ($1-\mu$) is higher compared to the chance of inter-community links ($\mu$). We created LRF benchmark networks with varying the value of $\mu$ in the range [0.1, 0.5] to investigate the impact of underlying community structure of a network on our method.  

\subsubsection{Real-world data.} \autoref{tab:data} gives a summary of the seven real-world network data sets we use. In citation networks, if a paper $i$ cites another paper $j$, the network contains an undirected edge connecting paper $i$ and paper $j$. Similarly, co-authorship networks represent authors as nodes and two authors are connected if they have published at least one paper together. Nodes of the network Enron-email are email addresses of Enron employees. If user $i$ has sent at least one email to the user $j$, nodes $i$ and $j$ are connected by an undirected edge. Twitter data set is made of 1000 ego-networks consisting of 4,869 Twitter lists \cite{leskovec2012learning}. 
Epinions, and Slashdot can be considered as web of trust networks. Even though Epinion and Slashdot networks are often labeled as online social networks, they differ from the usual notion of social networks as they represent who-trust-whom data of users instead of the relationships or interaction among users. In these networks, a user tags another user as trustworthy or not. They are sparse compared to online social networks.

\begin{table*}
 \caption{Description of data sets. (CA = co-authorship)}
 \label{tab:data}
 \begin{tabular}{ l r r r r r r r r }
  \toprule
  \ & HepPh & HepTh & Epinions  & Twitter & Stanford & AstroPh & DBLP & Slashdot \\ 
   \midrule
   Type  & citation & citation & web & social & web & CA & CA & web \\
    Nodes  & 34,546 & 27,770 & 75,789 & 81,306 & 281,903 & 18,772 &  317,080 & 82,168 \\
   Edges  & 421,578 & 352,807 & 508,837 & 1,768,149 & 2,312,497 &  198,110 & 1,049,866 & 549,202 \\
   Avg Clustering  & 0.2848 & 0.3120 & 0.1378 & 0.5653 & 0.5976 & 0.6306 & 0.6324 & 0.0603 \\ 
   \bottomrule
 \end{tabular}
\end{table*}

\subsection{Impact of Initial Sampling Method }
To investigate how the sampling method used to acquire the initial sample  influence the probing methods, we generate graph samples using two sampling methods. These are the methods we use:
\begin{enumerate}
\item Random node sampling (RN): At each step we choose one neighbor of a node already in the sample. 
\item Breadth-first search (BFS): Nodes are added to the sample in the order they are observed.
\end{enumerate}
 
\subsection{Methods}
We compare the performance of our algorithm against the following algorithms.

\subsubsection{Algorithms that do not use node features}

\begin{itemize}
	\item \textbf{Random walk (RW).} In this trivial baseline, we select one of the candidate nodes randomly for probing. This is equivalent to running our Bandit Explorer algorithm with only one cluster and using the random strategy for node selection.
	\item \textbf{Maximum observed degree (MOD).} This greedy method proposed in \cite{avrachenkov2014pay} is the current state-of-the-art algorithm for finding the network cover in an online manner. 
\end{itemize} 
%

\subsubsection{Algorithms that use node features}

\begin{itemize}
	\item \textbf{Lin-UCB.} This applies the UCB algorithm by \textcite{Dani2008} assuming that the reward of an arm is linearly dependent on its feature vector.
	\item \textbf{KNN-greedy}. This algorithm chooses the arm corresponding to the largest expected reward calculated by k-NN model.
	\item \textbf{KNN-$\epsilon$-greedy.} This algorithm chooses a random arm with probability $\epsilon$ while selecting the arm with  k-NN regression selects the arm rest of the times.
	\item \textbf{$i$KNN-UCB} This is our proposed algorithm, \autoref{alg:bandit}.
\end{itemize}  

\section{Results} 
\label{sec:results}

\subsection{Analysis on Synthetic Networks}
We probe incomplete BA and LFR networks obtained by RN and BFS sampling for 1,000 iterations ($T=1000$). Number of nodes observed in the BA network is shown in Figure \ref{fig:BA_50k_20}. For all networks generated by Barabasi-Albert (BA) model, MOD could observe more nodes than bandit algorithm. This confirms \textcite{avrachenkov2014pay}'s claim that MOD probing can achieve the best connected network cover for networks generated by preferential attachment processes.

\begin{figure}[ht]
	\begin{minipage}[b]{0.5\linewidth}
		\centering
		\includegraphics[width=\textwidth]{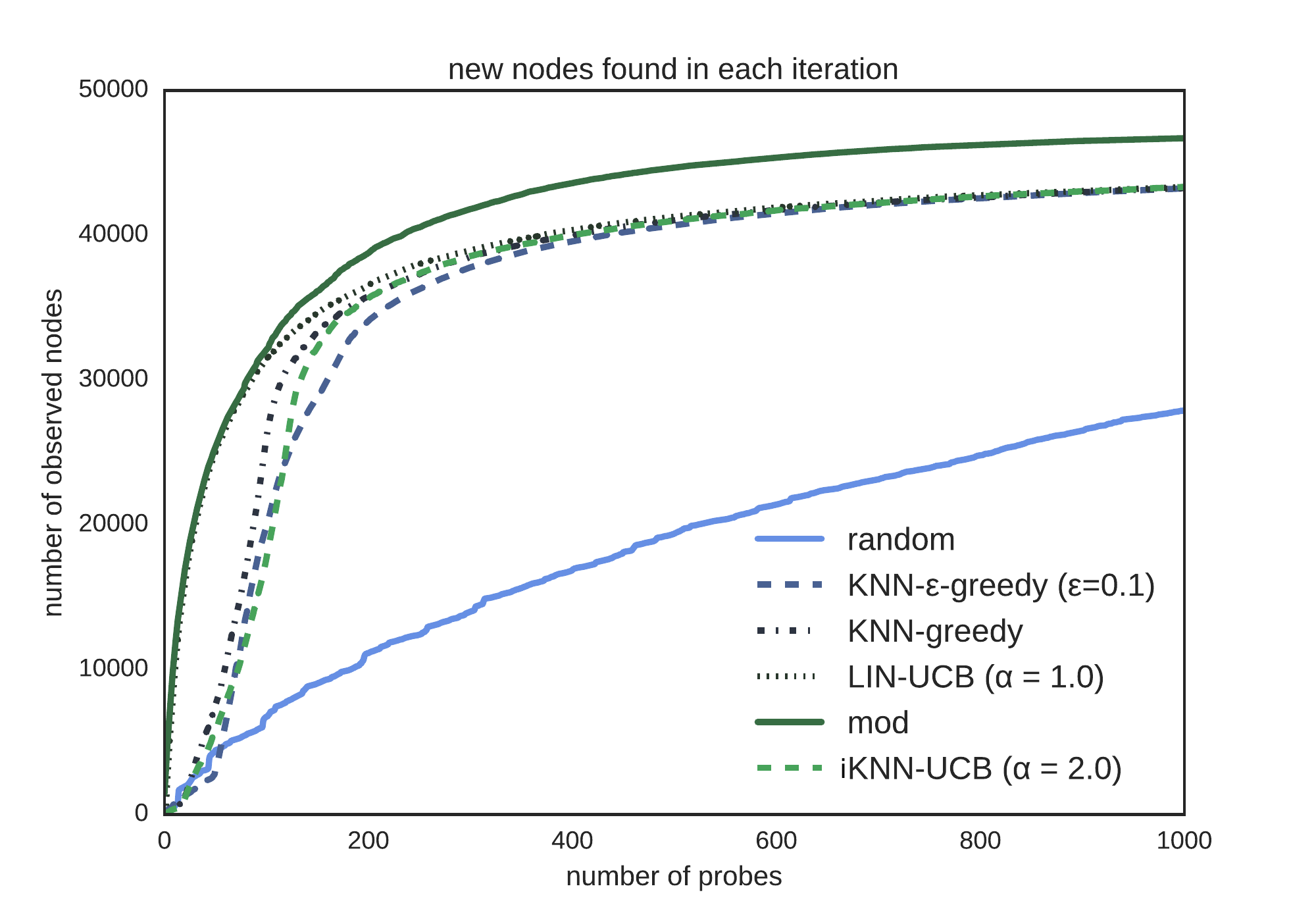}
		\small{(a)}
	\end{minipage}
	\begin{minipage}[b]{0.5\linewidth}
		\centering
		\includegraphics[width=\textwidth]{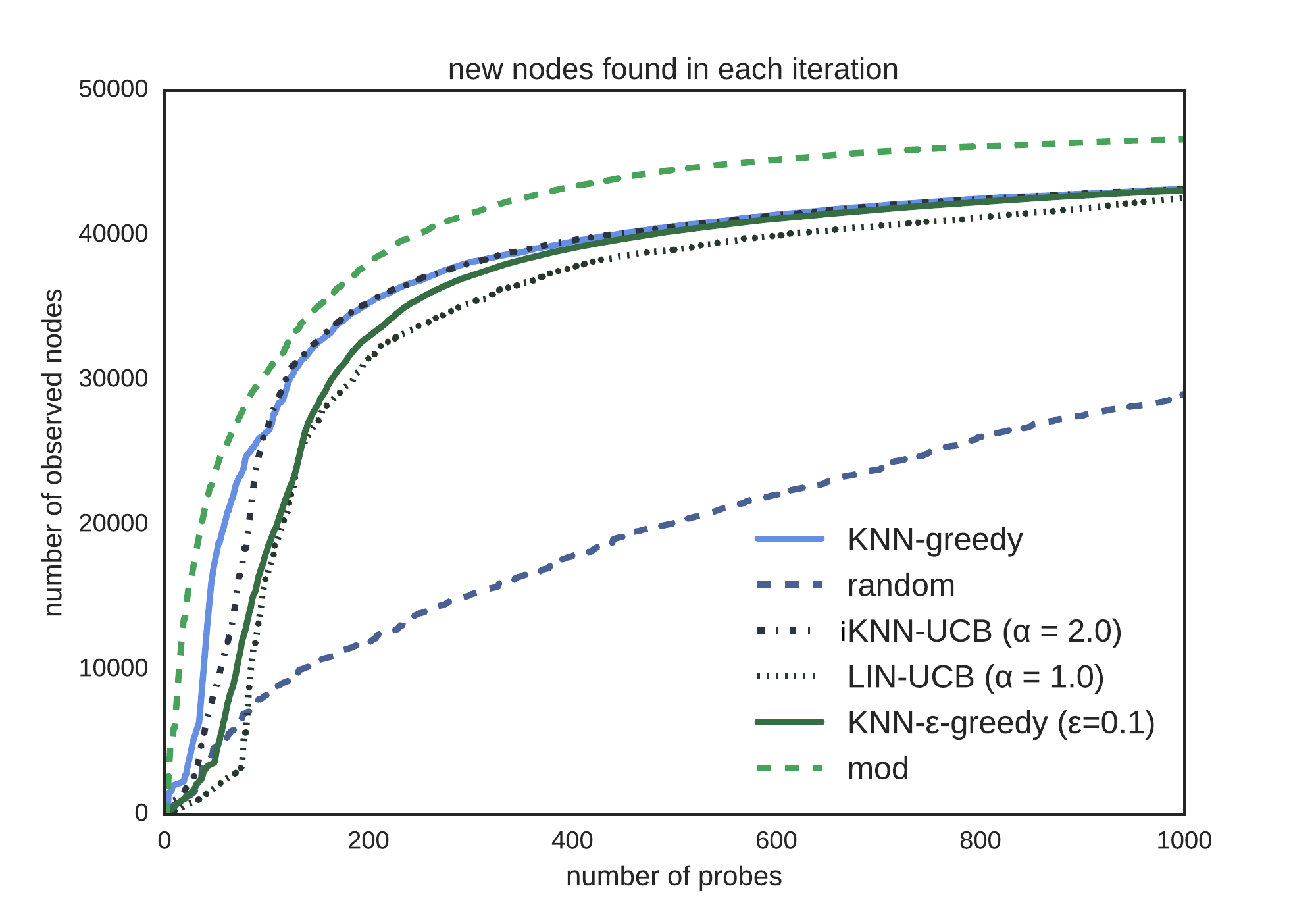}
		\small{(b)}
	\end{minipage}
	\caption{Scale-free network created by Barabasi-Albert model. (nodes=50,000, m = 20) (a) random node (RN) sample (b) BFS sample }
	\label{fig:BA_50k_20}
\end{figure}

\begin{figure}[ht]
	\begin{minipage}[b]{0.5\linewidth}
		\centering
		\includegraphics[width=\textwidth]{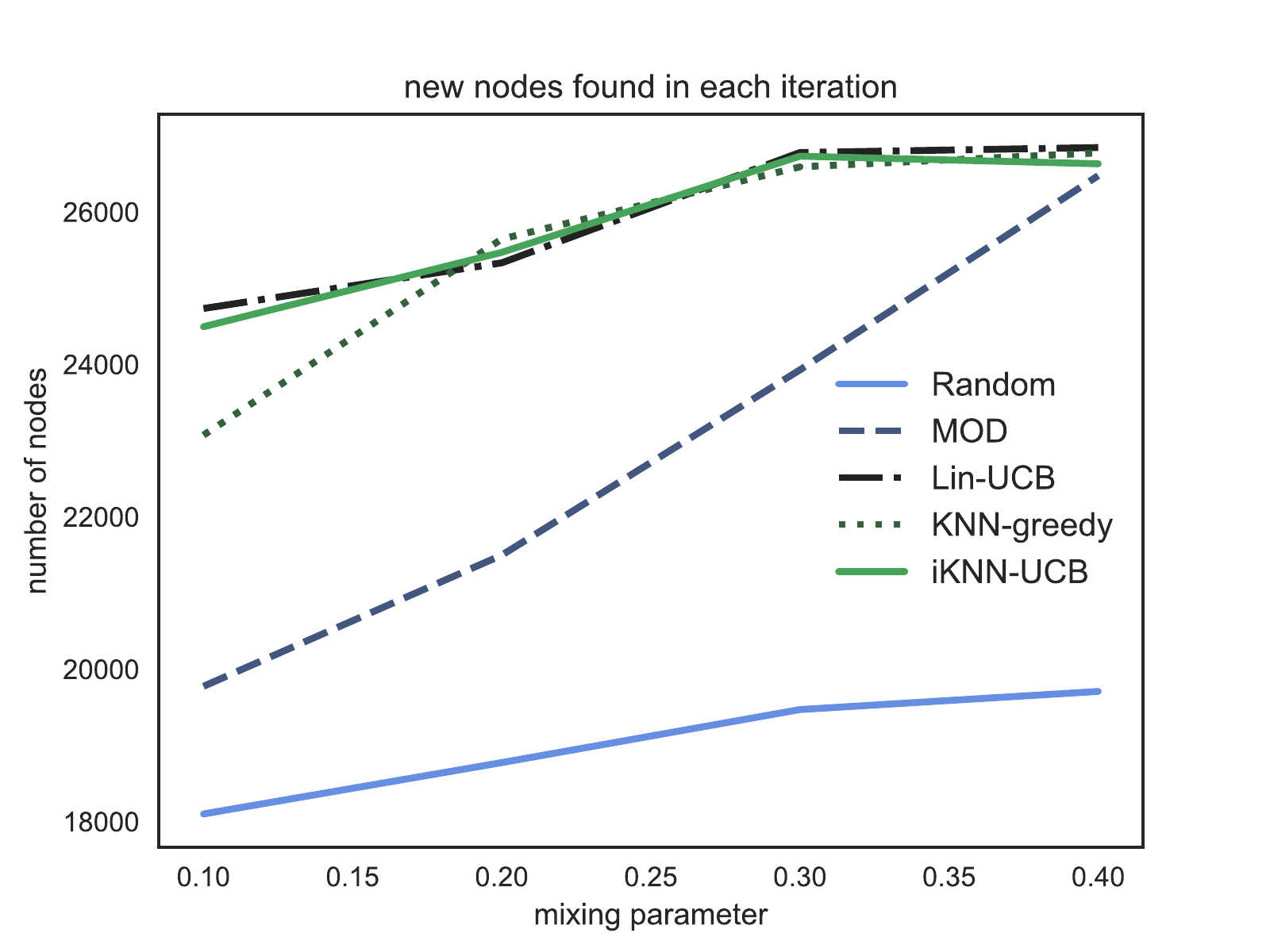}
		\small{(a)}
	\end{minipage}
	\begin{minipage}[b]{0.5\linewidth}
		\centering
		\includegraphics[width=\textwidth]{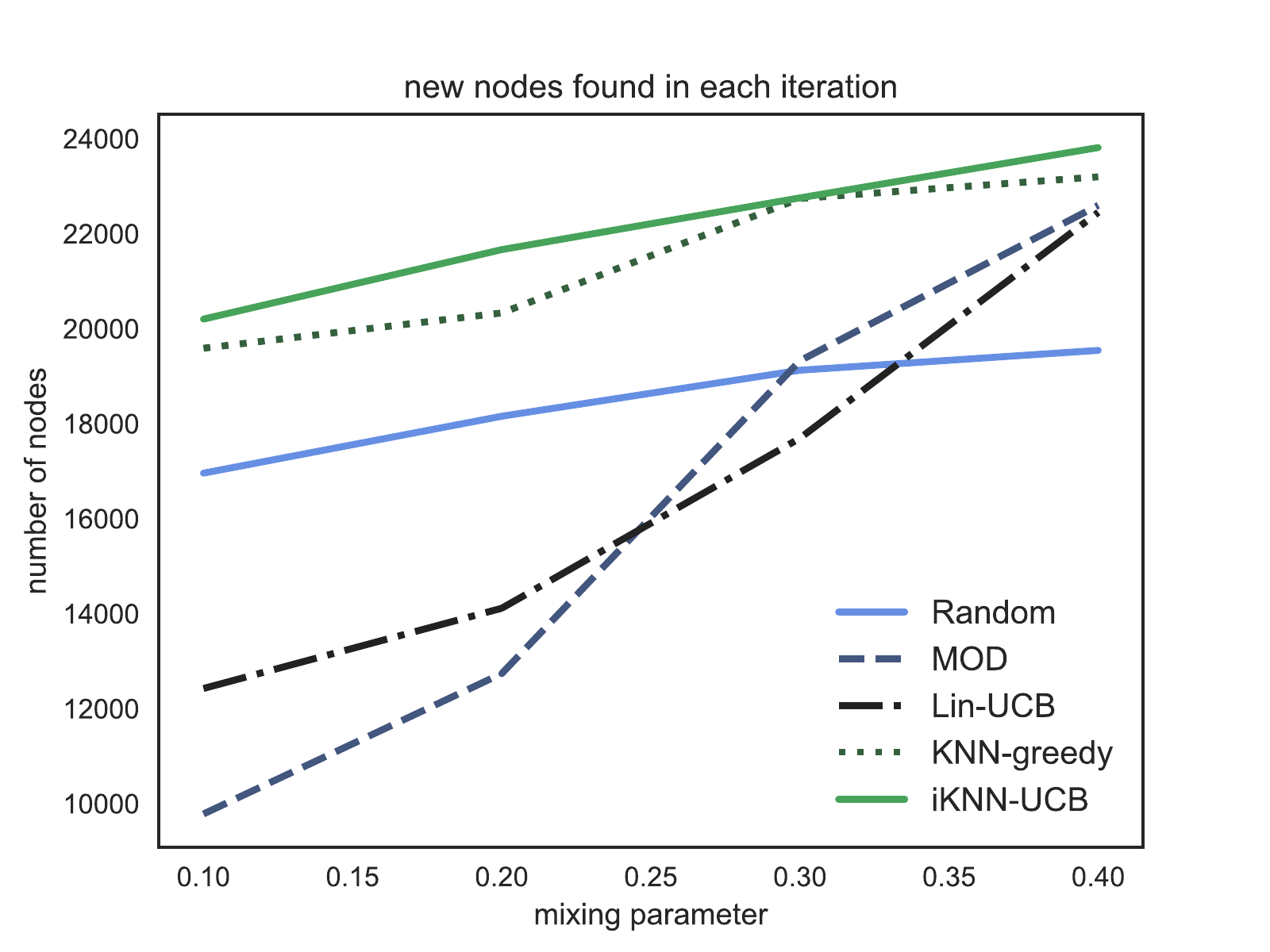}
		\small{(b)}
	\end{minipage}
	\caption{Performance on synthetic networks generated by LFR benchmark (a) RN sample (b) BFS sample}
	\label{fig:LFR}
\end{figure}

To understand how the existence of community structure impacts the probing, we evaluate the performance of all algorithms on synthetic networks generated by different configurations of LFR benchmark model \cite{lancichinetti2008benchmark}. We vary the mixing parameter $\mu$ from 0.1 to 0.5 keeping all other parameters of the model constant ($\gamma = 3$, $\beta = 1.3$, average degree = 25). $i$KNN-UCB significantly outperforms the baseline for networks with smaller $\mu$. When the initial sample is obtained by BFS sampling, $i$KNN-UCB outperforms all baselines by a significant margin. The gap between $i$KNN-UCB and the baseline is larger when the mixing parameter is small, network has significant community structure.
The experimental results on synthetic networks suggest that $i$KNN-UCB algorithm can adapt for  incomplete networks obtained by different sampling techniques and networks with structural properties such as community structure.

\subsection{Results on Real World Networks}
We use 8 real-world networks mentioned in \autoref{tab:data} and generate RN and BFS samples containing 5\% nodes of the original network $G$. Then 1,000 probing steps are performed. We perform each experiment five times initialized with different random seeds and report the average number of additional nodes which were observed in \autoref{fig:real_rn} and \autoref{fig:real_bfs}.

\begin{figure}
	\centering
	\includegraphics[width=\linewidth]{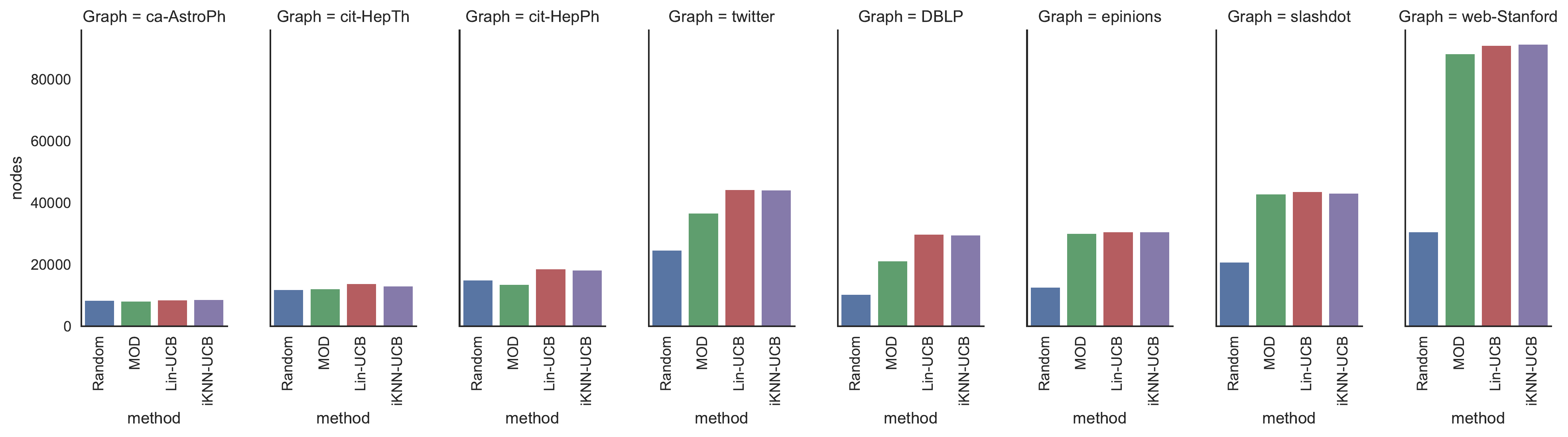}
	\caption{Comparison against baselines: 1000 probes run on 5\% nodes of each network. Each sample is created by performing a random walk on the original network}
	\label{fig:real_rn}
\end{figure}

\begin{figure}
	\centering
	\includegraphics[width=\linewidth]{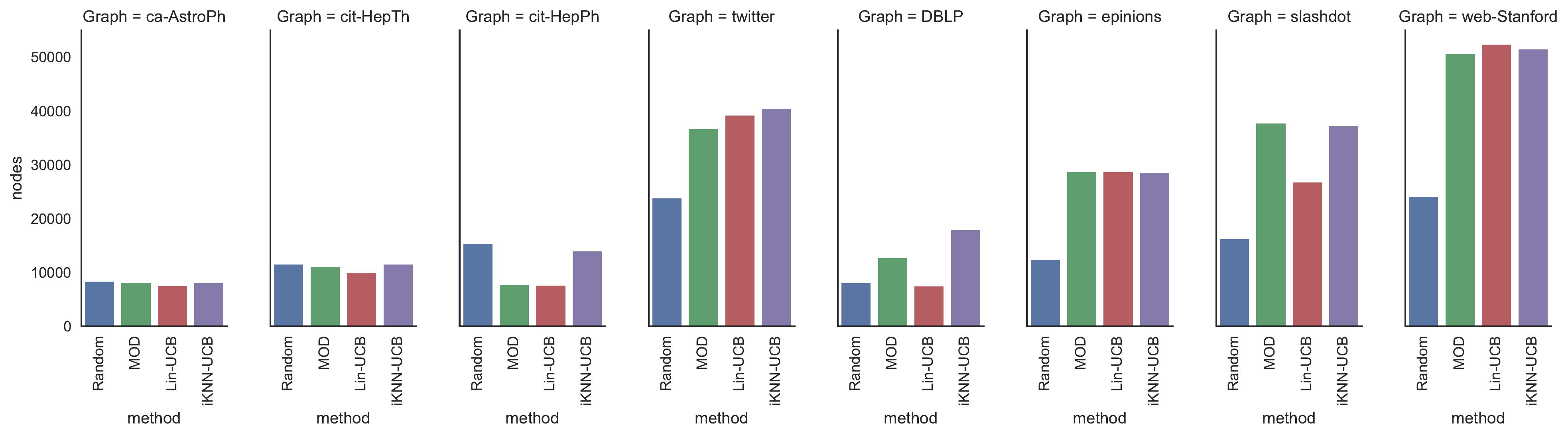}
	\caption{Comparison against baselines: 1000 probes run on 5\% nodes of each network. Each sample is created by performing a breadth first walk on the original network}
	\label{fig:real_bfs}
\end{figure}

$i$KNN-UCB and Lin-UCB bandit algorithms outperform all baseline methods in networks generated by both RN and BFS sampling. Even though Lin-UCB bandit algorithm observes as much nodes as $i$KNN-UCB for RN samples, its performance is worse for BFS samples. This shows that linear model in Lin-UCB is not capable of learning the relationship between observed node features and the true degree of a node if the sample is constructed by a BFS.

\section{Conclusions}
\label{sec:conclusions}
In this paper, we introduced a bandit based exploration algorithm for partially observed incomplete networks. We proposed a novel nonparametric multi-armed bandit algorithm $i$KNN-UCB with sublinear regret. Compared to existing solutions for the  \textit{Adaptive Graph Exploring} problem, the proposed method does not depend on a specific heuristic. Additionally, $i$KNN-UCB bandit algorithm outperforms the baseline methods irrespective of how the initial incomplete network is obtained. We provided experimental evidence for our approach using synthetic networks and variety of real-world networks. Using different configurations  of LFR benchmark networks, we observed that our algorithm outperforms all other baselines significantly when the network exhibits community structure prominently. Since the reward function is independent from the probing procedure, it is easy to define a new reward function to solve a different graph exploration problem (eg. finding a particular type of nodes).  

In this problem, we assumed that probing a node would reveal all its neighboring nodes. However in some real-world scenarios, only a certain number of neighbors is revealed (e.g., follower limit in Twitter API \footnote{https://dev.twitter.com/rest/reference/get/followers/ids}). As future work, we would explore how this current approach can be changed for such different settings of the same problem. 

